\documentclass[10pt,twocolumn,letterpaper]{article}

\usepackage{tikz}
\usepackage{tikz-cd}
\usetikzlibrary{matrix, calc, arrows}
\usepackage{tabulary}
\usepackage{booktabs}
\usetikzlibrary{tikzmark}

\usepackage{array}
\usepackage{iccv}
\usepackage{times}
\usepackage{graphicx}
\usepackage{amsmath}
\usepackage{amssymb}
\usepackage{amsthm}
\usepackage{algorithm}
\usepackage[noend]{algpseudocode}
\usepackage{authblk}

\DeclareMathOperator*{\argmin}{arg\!\min}
\DeclareMathOperator{\E}{\mathbb{E}}

\usepackage{array}
\newcolumntype{L}[1]{>{\raggedright\arraybackslash}p{#1}}
\newcolumntype{C}[1]{>{\centering\arraybackslash}p{#1}}
\newcolumntype{R}[1]{>{\raggedleft\arraybackslash}p{#1}}

\newcommand{\Id}{\textnormal{Id}}

\newcommand{\disc}{\textnormal{disc}}

\newtheorem{definition}{Definition}
\newtheorem{theorem}{Theorem}

\usepackage{url}
\usepackage{placeins}

\iccvfinalcopy 

\def\iccvPaperID{1661} 

\ificcvfinal\fi
\begin{document}

\title{Unsupervised Creation of Parameterized Avatars}
\author[1,2]{Lior Wolf}
\author[1]{Yaniv Taigman}
\author[1]{Adam Polyak}
\affil[1]{Facebook AI Research} 
\affil[2]{School of Computer Science, 
Tel Aviv University}

\tikzstyle{b} = [rectangle, draw, fill=blue!20, node distance=3cm, text width=6em, text centered, rounded corners, minimum height=4em, thick]
\tikzstyle{c} = [rectangle, draw, inner sep=0.5cm, dashed]
\tikzstyle{l} = [draw, -latex',thick]

\maketitle

\begin{abstract}
We study the problem of mapping an input image to a tied pair consisting of a vector of parameters and an image that is created using a graphical engine from the vector of parameters. The mapping's objective is to have the  output image as similar as possible to the input image. During training, no supervision is given in the form of matching inputs and outputs.

This learning problem extends two literature problems: unsupervised domain adaptation and cross domain transfer. We define a generalization bound that is based on discrepancy, and employ a GAN to implement a network solution that corresponds to this bound. Experimentally, our method is shown to solve the problem of automatically creating  avatars.
\end{abstract}
\section{Introduction}

The artist Hanoch Piven creates caricatures by arranging household items and scrap material in a frame and photographing the result, see Fig.~\ref{fig:piven}(a). How can a computer create such images? Given a training set consisting of Piven's images, Generative Adversarial Networks (GANs) can be used to create images that are as indistinguishable as possible from the training set. However, common sense tells us that for any reasonably sized training set, without knowledge about the physical world, the generated images would be easily recognized by humans as being synthetic. 
\begin{figure}[t]
\centering
\begin{tabular}{cc}
 \includegraphics[trim=0 0 0 0, clip, height=.499299\linewidth]
{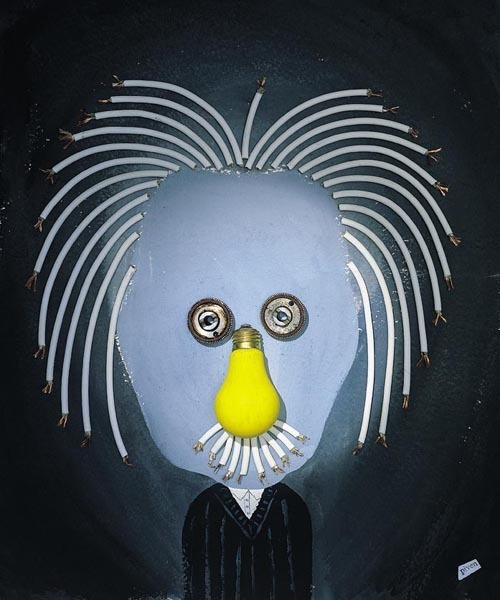}& 
 \includegraphics[trim=0 0 0 0, clip, height=.499299\linewidth]{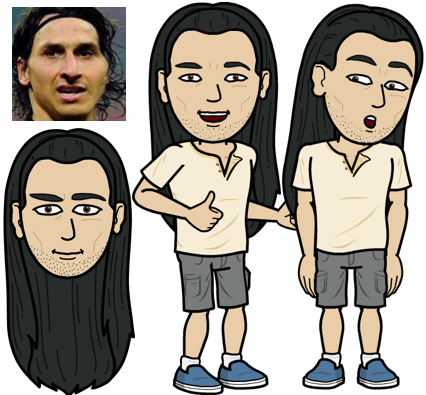}\\
 (a)& (b)\\
 \multicolumn{2}{c}{\includegraphics[trim=0 0 0 0, clip, width=.9499299\linewidth]{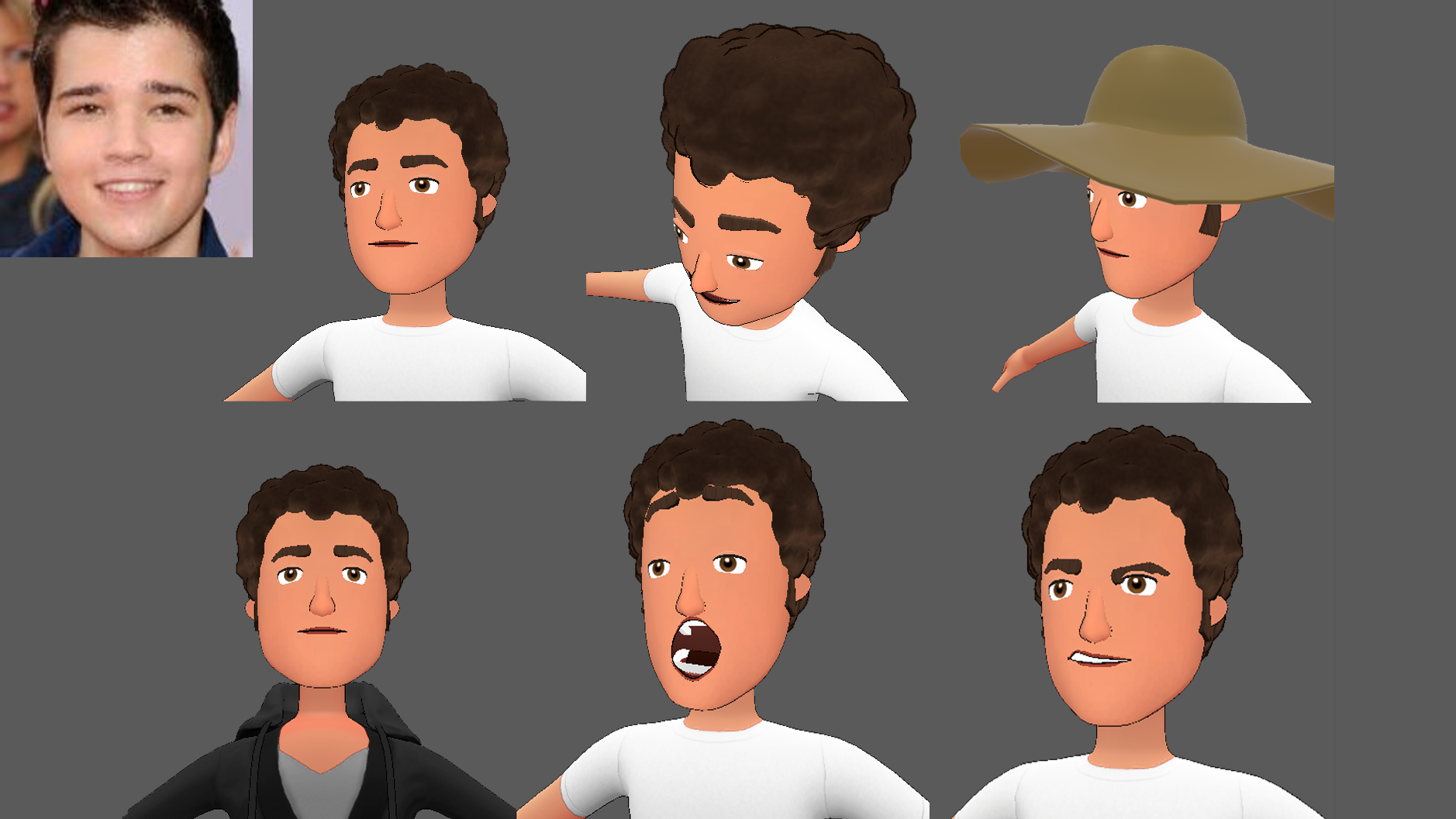}}\\
 \multicolumn{2}{c}{(c)}
 \end{tabular}
\caption{\label{fig:piven} (a) A caricature by Hanoch Piven. (b) From the image on the top left, our method computes the parameters of the face caricature below it, which can be rendered at multiple views and with varying expressions by the computer graphics engine. (c) Similarly, for 3D VR avatars.}
\end{figure}
As a second motivating example, consider the problem of generating computer avatars based on the user's appearance. In order to allow the avatars to be easily manipulated, each avatar is represented by a set of ``switches'' (parameters) that select, for example, the shape of the nose, the color of the eyes and the style of hair, all from a predefined set of options created by artists. Similar to the first example, the visual appearance of the avatar adheres to a set of constraints. Once the set of parameters is set, the avatar can be rendered in many variations  (Fig.~\ref{fig:piven}(b)).

The goal of this work is to learn to map an input image to two tied outputs: a vector in some parameter space and the image generated by this vector. While it is sufficient to recover just the vector of parameters and then generate the image, a non-intuitive result of our work is that it is preferable to recover the analog image first. In any case, the mapping between the input image and either of the outputs should be learned in an unsupervised way due to the difficulty of obtaining supervised samples that map input images to parameterized representations. In avatar creation, it is time consuming for humans to select the parameters that represent a user, even after considerable training. The selected parameters are also not guaranteed to be the optimal depiction of that user. Therefore, using unsupervised methods is both more practical and holds the potential to lead to more accurate results. 

In addition, humans can learn to create parameterized analogies without using matching samples.  Understanding possible computational processes is, therefore, an objective of AI, and is the research question addressed. Our contributions are therefore as follows: (i) we present a highly applicable and, as far as we know, completely unexplored vision problem; (ii)  the new problem is placed in the mathematical context of other domain shift problems; (iii) a generalization bound for the new problem is presented; (iv) an algorithm that matches the terms of the generalization bound is introduced; (v) the qualitative and quantitative success of the method further validates the non-intuitive path we take and (vi) the new method is shown to solve the parameterized avatar creation problem.

\subsection{Background}
\label{sec:relatedwork}

\noindent{\bf Generative Adversarial Networks} GAN~\cite{gan} methods train a generator network $G$ that synthesizes samples from a target distribution, given noise vectors, by jointly training a second network $d$. The specific generative architecture we employ is based on the architecture of~\cite{dcgan}. Since the image we create is based on an input and not on random noise, our method is related to Conditional GANs, which employ GANs in order to generate samples from a specific class~\cite{mirza2014conditional}, based on a textual description~\cite{RAYLLS16}, or to invert mid-level network activations~\cite{deepsim}. The CoGAN method~\cite{cogan}, like our method, generates a pair of tied outputs. However, this method generates the two based on a random vector and not on an input image. More importantly, the two outputs are assumed to be similar and their generators (and GAN discriminators) share many of the layers. In our case, the two outputs are related in a different way: a vector of parameters and the resulting image. The solutions are also vastly different. 

A recent work, which studied the learning of 3D structure from images in an unsupervised manner, shares some of computational characteristics with our problem~\cite{NIPS2016_6600}. The most similar application to ours, involves a parametrization of a 3D computer graphics object with 162 vertices, each moving along a line, a black-box camera projecting from 3D to 2D and a set of 2D images without the corresponding 3D configuration. The system then learns to map 2D images to the set of vertices. This setting shares with us the existence of a fixed mapping from the vector of parameters to the image. In our case, this mapping is given as a neural network that will be termed $e$, in their case, it is given as a black box, which, as discussed in Sec.~\ref{sec:conclusions} is a solvable challenge. A more significant difference is that in their case, the images generated by the fixed mapping are in the same domain as the input, while in our case it is from a different domain. The method employed in~\cite{NIPS2016_6600} completely differs from ours and is based on sequential generative models~\cite{draw}.

\noindent{\bf Distances between distributions}
In unsupervised learning, where one cannot match between an input sample and its output, many methods rely on measuring distances between distributions. Specifically, GANs were recently shown~\cite{Ganin:2016:DTN:2946645.2946704} to implement the theoretical notion of discrepancies. 
\begin{definition}[Discrepancy distance]\label{def:discrepancy}
Let $\mathcal{C}$ be a class of functions from $A$ to $B$ and let $\ell: B\times B  \rightarrow \mathbb{R}_+$ be a loss function over $B$. The discrepancy distance $\disc_{\mathcal{C}}$ between two distributions $D_1$ and $D_2$ over $A$ is defined as 
$\disc_{\mathcal{C}}(D_1,D_2) = \sup_{c_1,c_2 \in \mathcal{C}}\Big\vert R_{D_1}[c_1,c_2]-R_{D_2}[c_1,c_2]\Big\vert$, where $R_{D}[c_1,c_2] = \mathbb{E}_{x\sim D}\left[\ell(c_1(x),c_2(x))\right]$.
\end{definition}

\noindent{\bf Image synthesis with CNNs} The supervised network of~\cite{DBLP:conf/cvpr/DosovitskiySB15} receives as input a one-hot encoding of the desired model as well as view parameters and a 3D transformation and generates the desired view of a 3D object. 

DC-IGN~\cite{NIPS2015_5851} performs a similar task with less direct supervision. 
The training set of this method is stratified but not necessarily fully labeled and is used to disentangle the image representation in an encoder-decoder framework. Pix2pix~\cite{pix2pix} maps an image to another domain. This methods is fully supervised and requires pairs of matching samples from the two domains. 

\noindent{\bf Style transfer} In these methods~\cite{styletransfer,ulyanov16texture,Johnson2016Perceptual}, new images are synthesized by minimizing the content loss with respect to one input sample and the style loss with respect to one or more input samples. The content loss is typically the encoding of the image by a network training for an image categorization task, similar to our work. The style loss compares the statistics of the activations in various layers of the neural network. We do not employ style losses in our method and more significantly, the problem that we solve differs. This is not only because style transfer methods cannot capture semantics~\cite{02200}, but also because the image we generate has to adhere to specific constraints. Similarly, the work that has been done to automatically generate sketches from images, e.g.,~\cite{wang2013transductive,Zhang2016}, does not apply to our problem since it does not produce a parameter vector in a semantic configuration space. The literature of face sketches also typically trains in a supervised manner that requires correspondences between sketches and photographs.

\section{Problem Formulation}
\label{ref:problemformulation}
Problems involving domain shift receive an increasing amount of attention, as the field of machine learning moves its focus away from the vanilla supervised learning scenarios to new combinations of supervised, unsupervised and transfer learning. In this section, we formulate the new computational problem that we pose ``Tied Output Synthesis'' (TOS) and put it within a theoretical context. In the next section, we redefine the problem as a concrete deep learning problem. In order to maximize clarity, the two sections are kept as independent as possible.

\subsection{Related Problems}
\begin{figure*}[t]
\centering
\begin{tabular}{ccc}
\begin{tabular}{|c|cc|}
\hline
      & Input $\mathcal X$ & Output $\mathcal Y$\\
      \hline
    $1^{st}$  &   \begin{small}$\{x_i\sim D_T\}$\end{small}& \\
    $2^{nd}$  &   \begin{small}$\{x_j\sim D_S\}$\end{small} & \begin{small}$\{y_S(x_j)\}$\end{small} \\ 

\hline
  \end{tabular}
&
  \begin{tabular}{|c|cc|}
\hline
       & Input $\mathcal X$ & Output $\mathcal Y$\\
      \hline
    $1^{st}$  &   \begin{small}$\{x_i\sim D_1\}$ \end{small}&  \\ 
    $2^{nd}$  &   & \begin{small}$\{y(x_j)| x_j\sim D_2\}$\end{small}\\
    \hline
  \end{tabular}
 &   \begin{tabular}{|c|ccc|}
\hline
       & Input $\mathcal X$& Out. $\mathcal Y_1$ & Out. $\mathcal Y_2$\\
      \hline
    $1^{st}$  &   \begin{small}$\{x_i\sim D_1\}$ \end{small}& & \\ 
    $2^{nd}$  & & $e(c_j)$  & \begin{small}$\{c_j\sim D_2\}$\end{small}\\
    \hline
  \end{tabular}
  \\
  (a) & (b) & (c)

\end{tabular}
\caption{\label{fig:illustration} The domain shift configurations discussed Sec.~\ref{ref:problemformulation}. (a) The unsupervised domain adaptation problem. The algorithm minimizes the risk in a target domain using training samples $\{(x_i\sim D_S,y_S(x_i))\}^{m}_{i=1}$ and $\{x_i\sim D_T\}^{n}_{i=1}$. (b) The unsupervised domain transfer problem. In this case, the algorithm learns a function $G$ and is being tested on $D_1$. The algorithm is aided with two datasets: $\{x_i \sim D_1\}^{m}_{i=1}$ and $\{y(x_j) \sim D^y_2\}^{n}_{j=1}$. For example, in the facial emoji application $D_1$ is the distribution of facial photos and $D_2$ is the (unseen) distribution of faces from which the observed emoji were generated. (c) The tied output synthesis problem, in which we are give a set of samples from one input domain $\{x_i\sim D_1\}$, and matching samples from two tied output domains: $\{(e(c_j),c_j) | c_j\sim D_2\}$.}
\end{figure*}

In the {\bf unsupervised domain adaptation} problem~\cite{Crammer:2008:LMS:1390681.1442790,DBLP:conf/colt/MansourMR09,DBLP:journals/ml/Ben-DavidBCKPV10}, the algorithm trains a hypothesis on a source domain and the hypothesis is tested on a different target domain. The algorithm is aided with a labeled dataset of the source domain and an unlabeled dataset of the target domain. The conventional approach to dealing with this problem is to learn a feature map that (i) enables accurate  classification in the source domain and (ii) captures meaningful invariant relationships between the source and target domains. 

Let $\mathcal{X}$ be the input space and $\mathcal{Y}$ be the output space (the mathematical notation is also conveniently tabulated in the appendix). The source domain is a distribution $D_S$ over $\mathcal{X}$ along with a function $y_S:\mathcal{X} \rightarrow \mathcal{Y}$. Similarly, the target domain is specified by $(D_T,y_T)$. Given some loss function  $\ell: \mathcal{Y}\times\mathcal{Y} \rightarrow \mathbb{R}_+$ The goal is to fit a hypothesis $h$ from some hypothesis space $\mathcal{H}$, which minimizes the {\em Target Generalization Risk}, $R_{D_T}[h,y_T]$. Where a {\em Generalization Risk} is defined as 
\begin{small}
$R_{D}[h_1,h_2]=\mathbb{E}_{x\sim D}\left[\ell(h_1(x),h_2(x))\right]$.
\end{small}
The distributions $D_S$, $D_T$ and the target function $y_T:\mathcal{X} \rightarrow \mathcal{Y}$ are unknown to the learning algorithm. 
Instead, the learning algorithm relies on a training set of labeled samples $\{(x,y_S(x))\}$, where $x$ is sampled from $D_S$ as well as on an unlabeled training set of samples $x\sim D_T$, see Fig.~\ref{fig:illustration}(a).

In the {\bf cross domain transfer} problem, the task is to learn a function that maps samples from the input domain $\cal X$ to the output domain $\cal Y$. It was recently presented in~\cite{02200}, where a GAN based solution was able to convincingly transform face images into caricatures from a specific domain. 

The training data available to the learning algorithm in the cross domain transfer problem is illustrated in Fig.~\ref{fig:illustration}(b).  
The problem consists of two distributions, $D_1$ and $D_2$, and a target function, $y$. The algorithm has access to the following two unsupervised datasets:
$\{x_i {\sim} D_1\}^{m}_{i=1}$  and 
$\{y(x_j)| x_j {\sim} D_2\}^n_{j=1}$.  
The goal is to fit a function $h=g\circ f\in \mathcal{H}$ that optimizes 
$\inf_{h\in \mathcal{H}} R_{D_1}[h,y]$.

It is assumed that: (i) $f$ is a fixed pre-trained feature map and, therefore, $\mathcal{H} = \left\{g \circ f \big\vert g \in \mathcal{H}_2\right\}$ for some hypothesis class $\mathcal{H}_2$; and (ii) $y$ is idempotent, i.e, $y \circ y \equiv y$. For example, in~\cite{02200}, $f$ is the DeepFace representation~\cite{deepface} and $y$ maps face images to emoji caricatures. In addition, applying $y$ on an emoji gives the same emoji. 
Note that according to the terminology of~\cite{02200}, $D_1$ and $D_2$ are the source and target distributions respectively. However, the loss $R_{D_1}[h,y]$ is measured over $D_1$, while in domain adaptation, it is measured over the target distribution.

Recently~\cite{gw2017}, the cross domain transfer problem was analyzed using the theoretical term of discrepancy. Denoting, for example, $y \circ D$ to be the distribution of the $y$ mappings of samples $x\sim D$, then the following bound is obtained. 
\begin{theorem}[Domain transfer~\cite{gw2017}]\label{thm:main3}
If $\ell$ satisfies the triangle inequality\footnote{For all $y_1,y_2,y_3\in \mathcal{Y}$ it holds that $\ell(y_1,y_3) \leq  \ell(y_1,y_2)+\ell(y_2,y_3)$. This holds for the absolute loss, and can be relaxed to the square loss, where it holds up to a multiplicative factor of 3.} and $\mathcal{H}_2$ (the hypothesis class of $g$) is a universal Lipschitz hypothesis class\footnote{A function $c\in \mathcal{C}$ is Lipschitz with respect to $\ell$, if there is a constant $L>0$ such that: $\forall a_1, a_2 \in A: \ell(c(a_1),c(a_2)) \leq L \cdot \ell(a_1,a_2) $. A hypothesis class $\mathcal{C}$ is universal Lipschitz with respect to $\ell$ if all functions $c\in\mathcal{C}$ are Lipschitz with some universal constant $L>0$. This holds, for example, for neural networks with leaky ReLU activations and weight matrices of bounded norms, under the squared or absolute loss.}, then for all $h=g\circ f \in \mathcal{H}$,
\begin{small}
\begin{equation}
\begin{aligned}
R_{D_1}[h,y] \leq &
R_{y\circ D_2}[h,\Id] +R_{D_1}[f \circ h,f]\\
&+ \disc_{\mathcal{H}}(y \circ D_2, h \circ D_1)+\lambda 
\end{aligned}
\end{equation}
\end{small}
Here, $\lambda = \min_{h\in \mathcal{H}} \left\{R_{y\circ D_2}[h,\Id] + R_{D_1}[h,y]\right\}$ and $h^*=g^*\circ f$ is the corresponding minimizer.
\end{theorem}

This theorem matches the method of~\cite{02200}, which is called DTN. It bounds the risk $R_{D_1}[h,y]$, i.e., the expected loss (using $\ell$) between the mappings by the ground truth function $y$ and the mapping by the learned function $h$ for samples $x\sim \mathcal D_1$. The first term in the R.H.S $R_{y\circ D_2}[h,\Id]$ is the $L_{\text{TID}}$ part of the DTN loss, which, for the emoji generation application, states that emoji caricatures are mapped to themselves. The second term $R_{D_1}[f \circ h,f]$ corresponds to the $L_{\text{CONST}}$ term of DTN, which states that the DeepFace representations of the input face image and the resulting caricature are similar. The theorem shows that his constancy does not need to be assumed and is a result of the idempotency of $y$ and the structure of $h$. The third term $\disc_{\mathcal{H}}(y\circ D_2, h \circ D_1)$ is the  GAN element of the DTN method, which compares generated caricatures ($h \circ D_1$) to the training dataset of the unlabeled emoji ($y\circ D_2$). Lastly, the $\lambda$ factor captures the complexity of the hypothesis class $\cal H$, which depends on the chosen architecture of the neural network that instantiates $g$. A similar factor in the generalization bound of the unsupervised domain adaptation problem is presented in~\cite{DBLP:journals/ml/Ben-DavidBCKPV10}.  

\subsection{The Tied Output Synthesis Problem}

The problem studied in this paper, is a third flavor of domain shift, which can be seen as a mix of the two problems: unsupervised domain adaptation and the cross domain transfer problem. Similar to the unsupervised domain transfer problem, we are given a set of supervised labeled samples. The samples $c_j$ are drawn i.i.d from some distribution $D_2$ in the space $\mathcal Y_2$ and are given together with their mappings $e(c_j) \in \mathcal Y_1$. In addition, and similar to the cross domain transfer problem, we are given samples $x_i\in \mathcal X$ drawn i.i.d from another distribution $D_1$. The goal is to learn a mapping $y:\mathcal X \rightarrow \mathcal Y_2$ that satisfies the following condition $y \circ e \circ y = y$. The hypothesis class contains functions $h$ of the form $c \circ g \circ f$ for some known $f$ for $g \in \mathcal H_2$ and for $c \in \mathcal H_3$. $f$ is a pre-learned function that maps the input sample in $\mathcal X$ to some feature space, $g$ maps from this feature space to the space $\mathcal Y_1$, and $c$ maps from this space to the space of parameters $\mathcal Y_2$, see Fig.~\ref{fig:illustration}(c) and Fig.~\ref{fig:TOSfig}.

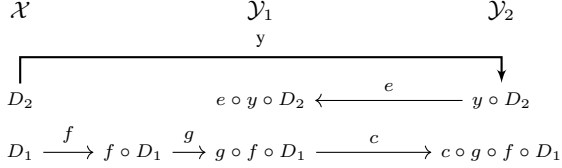
\begin{figure}[t]
\centering
\begin{tikzpicture}[auto,scale=1]
	\node (lab1) {\begin{small}$\mathcal X$\end{small}};
    \node (lab11) at ([shift={(3.2,0)}] lab1) {\begin{small}$\mathcal Y_1$\end{small}};
    \node (lab12) at ([shift={(6.4,0)}] lab1) {\begin{small}$\mathcal Y_2$\end{small}};

	\node (in1) at ([shift={(0,-1.2)}] lab1) {\begin{scriptsize}$D_2$\end{scriptsize}};
    \node (out11) at ([shift={(3.2,0)}] in1) {\begin{scriptsize}$e\circ y\circ D_2$\end{scriptsize}};
    \node (out12) at ([shift={(6.4,0)}] in1) {\begin{scriptsize}$y \circ D_2$\end{scriptsize}};
    
    \node (in2) at ([shift={(0,-.690)}] in1) {\begin{scriptsize}$D_1$\end{scriptsize}};
    \node (rep1) at ([shift={(1.5,-.690)}] in1) {\begin{scriptsize}$f \circ D_1$\end{scriptsize}};
    \node (out21) at ([shift={(3.2,-.690)}] in1) {\begin{scriptsize}$g \circ f \circ D_1$\end{scriptsize}};
    \node (out22) at ([shift={(6.4,-.690)}] in1) {\begin{scriptsize}$c \circ g \circ f \circ D_1$\end{scriptsize}};
    
    \draw[<-] (out11) -- (out12) node[midway] {\begin{scriptsize}$e$\end{scriptsize}};

    \draw[->] (in2) -- (rep1) node[midway] {\begin{scriptsize}$f$\end{scriptsize}};
    \draw[->] (rep1) -- (out21) node[midway] {\begin{scriptsize}$g$\end{scriptsize}};
    
    \draw[->] (out21) -- (out22) node[midway] {\begin{scriptsize}$c$\end{scriptsize}};
    
    \path [l] (in1.north) -- ++(0,0.35)  -- node[pos=0.5] {\begin{scriptsize}y\end{scriptsize}} ++(6.4,0) --  (out12.north);
    
\end{tikzpicture}
\caption{\label{fig:TOSfig} Tied Output Synthesis. The unknown function $y$ is learned by the approximation $h = c\circ g \circ f$. $f$ and $e$ are given. $D_1$ is the distribution of input images at test time. During training, we observe tied mappings $(y(x),e(y(x)))$ for unknown samples $x\sim D_2$ as well unlabeled samples from the other distribution $D_1$.}
\end{figure}

Our approach assumes that $e$ is prelearned from the matching samples $(c_j,e(c_j))$. However, $c$ is learned together with $g$. This makes sense, since while $e$ is a feedforward transformation from a set of parameters to an output, $c$ requires the conversion of an input of the form $g(f(x))$ where $x\sim D_1$, which is different from the image of $e$ for inputs in $\mathcal Y_2$. The theorem below describes our solution.
\begin{theorem}[Tied output bound]\label{thm:main4}
If $\ell$ satisfies the triangle inequality and $\mathcal{H}_2$ is a universal Lipschitz hypothesis class with respect to $\ell$, then for all $h=c\circ g\circ f \in \mathcal{H}$,
\begin{small}
\begin{equation}
\begin{aligned}
R_{D_1}[e\circ h,e\circ y] \leq &
R_{D_1}[e \circ h,g \circ f] + R_{e\circ y\circ D_2}[g \circ f,\Id] \\
& +R_{D_1}[f \circ g \circ f,f]\\
&+ \disc_{\mathcal{H}}(e \circ y \circ D_2, g \circ f \circ D_1)+\lambda, 
\end{aligned}
\end{equation}
\end{small}
where $\lambda = \min_{g \in \mathcal{H}_2} \left\{R_{e\circ y\circ D_2}[g \circ f,\Id] + R_{D_1}[g \circ f,e\circ y]\right\}$ and $g^*$ is the corresponding minimizer.
\end{theorem}
\begin{proof}
By the triangle inequality, we obtain:\\
$R_{D_1}[e\circ h,e\circ y] \leq 
R_{D_1}[e \circ h,g \circ f] + R_{D_1}[g \circ f,e \circ y]$. 

Applying Thm.~\ref{thm:main3} completes the proof:
\begin{small}
\begin{equation*}
\begin{aligned}
R_{D_1}[g\circ f,e\circ y] \leq &
R_{e\circ y\circ D_2}[g \circ f,\Id] +R_{D_1}[f \circ g \circ f,f]\\
&+ \disc_{\mathcal{H}}(e \circ y \circ D_2, g \circ f \circ D_1)+\lambda \qedhere
\end{aligned}
\end{equation*} 
\end{small}
\end{proof}

Thm.~\ref{thm:main4} presents a recursive connection between the tied output synthesis problem and the cross domain transfer problem. This relation can be generalized for tying even more outputs to even more complex relations among parts of the training data. The importance of having a generalization bound to guide our solution stems from the plausibility of many other terms such as $R_{e \circ y \circ D_2}[e \circ h,g \circ f]$ or $R_{D_1}[f \circ g \circ f,f \circ e \circ h]$.

\noindent{\bf Comparing to Unsupervised Cross Domain Transfer} The tied output problem is a specific case of cross domain transfer with $\mathcal Y$ of the latter being $\mathcal Y_1 \times \mathcal Y_2$ of the former. However, this view makes no use of the network $e$. Comparing Thm.~\ref{thm:main3} and Tmm.~\ref{thm:main4}, there is an additional term in the second bound: $R_{D_1}[e \circ h,g \circ f]$. It expresses the expected loss (over samples from $D_1$) when comparing the result of applying the full cycle of encoding by $f$, generating an image by $g$, estimating the parameters in the space $\mathcal Y_2$ using $c$, and synthesizing the image that corresponds to these parameters using $e$, to the result of applying the subprocess that includes only $f$ and $g$. 

\noindent{\bf Comparing to Unsupervised  Domain Adaptation} Consider the domain $\mathcal X \cup \mathcal Y_1$ and learn the function $e^{-1}$ from this domain to $\mathcal Y_2$, using the samples  $\{(e(c_j),c_j) | c_j\sim D_2\}$, adapted to $x_i~\sim D_1$. This is a domain adaptation problem with $D_S = e \circ D_2$ and $D_T = D_1$. Our experiments show that applying this reduction leads to suboptimal results. This is expected, since this approach does not make use of the prelearned feature map $f$. This feature map is not to be confused with the feature network learned in~\cite{Ganin:2016:DTN:2946645.2946704}, which we denote by $p$. The latter is meant to eliminate the differences between $p \circ D_S$ and $p \circ D_T$. However, the prelearned $f$ leads to easily distinguishable $f \circ D_S$ and $f \circ D_T$. 

The unsupervised domain adaptation and the TOS problem become more similar, if one identifies $p$ with the conditional function that applies $g \circ f$ to samples from $\mathcal X$ and the identity to samples from $\mathcal Y_1$. In this case, the label predictor of~\cite{Ganin:2016:DTN:2946645.2946704} is identified with our $c$ and the discrepancy terms (i.e., the GANs) are applied to the same pairs of distributions. However, the two solutions would still differ since (i) our solution minimizes $R_{D_1}[e \circ h,g \circ f]$, while in unsupervised domain adaptation, the analog term is minimized over $D_S = e\circ D_2$ and (ii) the additional non-discrepancy terms would not have analogs in the domain adaptation bounds.

\section{The Tied Output Synthesis Network} 
We next reformulate the problem as a neural network challenge. For clarity, this formulation is purposefully written to be independent of the mathematical presentation above. We study the problem of projecting an image in one domain to an image in another domain, in which the images follow a set of specifications. Given a domain, $\mathcal X$, a mapping $e$ and a function $f$, we would like to learn a generative function $G$ such that $f$ is invariant under $G$, i.e., $f\circ G = f$,  and that for all samples $x\in \mathcal X$, there exists a configuration $u \in \mathcal Y_2$ such that $G(x)=e(u)$.  Other than the functions $f$ and $e$, the training data is unsupervised and consists of a set of samples from the source domain $\mathcal X$ and a second set from the target domain of $e$, which we call $\mathcal Y_1$.

In comparison to the domain transfer method presented in~\cite{02200}, the domain $\mathcal Y_1$ is constrained to be the image of a mapping $e$. 
DTN cannot satisfy this requirement, since presenting it with a training set $\mathbf t$ of samples generated by $e$ is not a strong enough constraint. Furthermore, the real-world avataring applications require the recovery of the configuration $u$ itself, which allows the synthesis of novel samples using an extended engine $e^*$ that generates new poses, expressions in the case of face images, etc. 
\subsection{The interplay between the trained networks}

In a general view of GANs, assume a loss function $\ell(G,d,x)$, for some function $d$ that receives inputs in the domain $\mathcal Y_1$. $G$, which maps an input $x$ to entities in $\mathcal Y_1$, minimizes the following loss: 
${L_\text{GAN}} = \max_d -\E_x \ell(G,d,x)$.
This optimization is successful, if {\bf for every} function $d$, the expectation of $\ell(G,d,x)$ is small for the learned $G$. It is done by maximizing this expectation with respect to $d$, and minimizing it with respect to $G$. The two learned networks $d$ and $G$ provide a training signal to each other.

Two networks can also provide a mutual signal by collaborating on a shared task. Consider the case in which $G$ and a second function $c$ work hand-in-hand in order to minimize the expectation of some other loss $\ell(G,c,x)$. In this case, $G$ ``relies'' on $c$ and minimizes the following expression:
\begin{equation}
\label{eq:collab}
L_{c} = \min_c \E_x \ell(G,c,x).
\end{equation}
This optimization succeeds if {\bf there exists} a function $c$ for which, post-learning, the expectation $\E_x \ell(G,c,x)$ is small. 

In the problem of tied output synthesis, the function $e$ maps entities $u$ in some configuration space $\mathcal Y_2$ to the target space $\mathcal Y_1$. $c$ maps samples from $\mathcal Y_1$ to the configuration space, essentially inverting $e$. The suitable loss is:
\begin{equation}
\label{eq:Le}
{\ell}_e(G,c,x) = \|G(x)-e(c(G(x)))\|^2.
\end{equation}

For such a problem, the optimal $c$ is given by 
$c^*(z) = \argmin_u \|z-e(u)\|^2$. 
This implicit function is intractable to compute, and $c$ is learned instead as a deep neural network.

\subsection{The complete network solution}
\label{sec:gcn}

The learning algorithm is given, in addition to two mappings $e$ and $f$, a training set $\mathbf s \subset \mathcal X$, and a training set $\mathbf t \subset \mathcal Y_1$. Similar to~\cite{02200}, we define $G$ to be composed out of $f$ and a second function $g$ that maps from the output space of $f$ to $T$, i.e., $G = g \circ f$. The $e$ compliance term ($L_{c}$ of Eq.~\ref{eq:collab} using $\ell_e$ of Eq.~\ref{eq:Le}) becomes:
\begin{equation}
\label{eq:cuger}
L_{c} = \sum_{x \in \mathbf s}\|g(f(x))-e(c(g(f(x))))\|^2
\end{equation}
In addition, we minimize $L_\text{CONST}$, which advocates that for every input $x \in \mathbf s$, $f$ remains unchanged as $G$ maps it to $\mathcal Y_1$:
\begin{equation}
\label{eq:mzconst}
L_{\text{CONST}} = \sum_{x \in \mathbf s} \|f(x)-f(G(x))\|^2
\end{equation}
\begin{figure}[t]
\includegraphics[width=\linewidth]{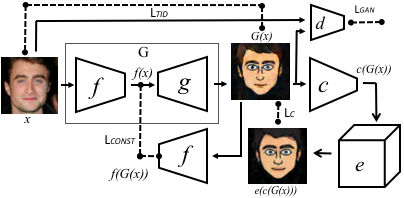}
\caption{\label{fig:DTfig} The training constraints of the Tied Output Synthesis method. The learned functions are $c$, $d$, and $G = g \circ f$, for a given $f$. The mapping $e$ is assumed to be known a-priori. Dashed lines denote loss terms.}
\end{figure}
A GAN term is added to ensure that the samples generated by $G$ are indistinguishable from the set $\mathbf t$. The GAN employs a binary classifier network $d$, and makes use of the training set $\mathbf t$.  Specifically, the following form of $\ell$ is used in $L_\text{GAN}$: 
\begin{equation}
\label{eq:mzgan}
\ell(G,d,x) = \log [1-d(G(x))] + \frac{1}{|\mathbf t |}\sum_{x' \in \mathbf t} \log [d(x')]. 
\end{equation}
Like~\cite{02200}, the following term encourages $G$ to be the identity mapping for samples from $\mathbf t$. 
\begin{equation}
\label{eq:tidz}
L_{\text{TID}} = \sum_{x \in \mathbf t} \|x-g(f(x))\|^2
\end{equation}
Taken together, $d$ maximizes $L_{GAN}$, and both $g$ and $c$ minimize
$L_{c} + \alpha L_\text{GAN} + \beta L_{\text{CONST}} + \gamma L_{\text{TID}} + \delta L_{\text{TV}}$ 
for some non-negative weights $\alpha,\beta,\gamma,\delta$, where $L_{\text{TV}}$, is the total variation loss, which smooths the resulting image $z=[z_{ij}]=G(x)$: 
$L_{TV}(z) =
 \sum_{i,j}
 \left(
 \left(z_{i,j+1} - z_{ij}\right)^2 +
 \left(z_{i+1,j} - z_{ij}\right)^2
 \right)^\frac{1}{2}$. The method is illustrated in Fig.~\ref{fig:DTfig} and laid out in Alg.~\ref{alg:euclid}.

\begin{algorithm}[t]
  \caption{The TOS training algorithm.}\label{alg:euclid}
  \begin{algorithmic}[1]
  \State {Given the function $e:\mathcal Y_2\rightarrow \mathcal Y_1$, an embedding function $f$, and $\mathbf S\subset \mathcal X$, $\mathbf T \subset \mathcal Y_1$ training sets.}
  \State {Initialize networks $c$, $g$ and $d$ }
  \algrenewcommand\algorithmicindent{0.80em}%
  \While{iter $<$ numiters}	
  \State{Sample mini-batches $\mathbf s \subset \mathbf S$, $\mathbf t \subset \mathbf T$}
  \State{Compute feed-forward $d(t)$, $d(g(f(s)))$ }
  \State{Update $d$ by minimizing $\ell(G,d,x)$ for $x\in \mathbf s$\Comment{Eq.~\ref{eq:mzgan}}}
  \State{Update $g$ by maximizing $\ell(G,d,x)$ for $x\in \mathbf s$\Comment{Eq.~\ref{eq:mzgan}}}
  \State{Update $g$ by minimizing $L_{\text{TID}}$ \Comment{Eq.~\ref{eq:tidz}} }
  \State{Update $g$ by minimizing $L_{\text{CONST}}$ \Comment{Eq.~\ref{eq:mzconst}} }
  \State{Update $g$ by minimizing $L_{\text{TV}}$}
  \State{Compute $e(c(z))$ by feed-forwarding $z:=g(f(s))$}
  \State{Update $c$ and $g$ by minimizing $L_{\text{c}} $ \Comment{Eq.~\ref{eq:cuger}} }  
  \EndWhile	
  \end{algorithmic}
\end{algorithm}
\FloatBarrier

In the context of Thm.~\ref{thm:main4}, the term $L_c$ corresponds to the risk term $R_{D_1}[e \circ h,g \circ f]$ in the theorem and compares samples transformed by the mapping $g\circ f$ to the mapping of the same samples to a configuration in $\mathcal Y_2$ using $c\circ g\circ f$ and then to $\mathcal Y_1$ using $e$. The term $L_\text{TID}$ corresponds to the risk $R_{e\circ y\circ D_2}[g \circ f,\Id]$, which is the expected loss over the distribution from which $\mathbf t$ is sampled, when comparing the samples in this training set to the result of mapping these by $g\circ f$. The discrepancy term $\disc_{\mathcal{H}}(e \circ y \circ D_2, g \circ f \circ D_1)$ matches the $L_\text{GAN}$ term, which as explained above, measures a distance between two distributions, in this case, $e \circ y \circ D_2$, which is the distribution from which the training set $\mathbf t$ is taken, and the distribution of mappings by $g\circ f$ of the samples $\mathbf s$ which are drawn from $D_1$.

\section{Experiments}

The Tied Output Synthesis (TOS) method is evaluated on a toy problem of inverting a polygon synthesizing engine and on avatar generation from a photograph for two different CG engines. The first problem is presented as a mere illustration of the method, while the second is an unsolved real-world challenge.

\subsection{Polygons}

The first experiment studies TOS in a context that is independent of $f$ constancy. Given a set of images $\mathbf t \in \mathcal Y_1$, and a mapping $e$ from some vector space to $\mathcal Y_1$, learn a mapping $c$ and a generative function $G$ that creates random images in $\mathcal Y_1$ that are $e$-compliant (Eq.~\ref{eq:Le}).

We create binary $64\times 64$ images of regular polygons by sampling uniformly three parameters: the number of vertices (3-6), the radius of the enclosing circle (15-30), and a rotation angle in the range $[-10,10]$. Some polygons are shown in Fig.~\ref{fig:poly}(a). 10,000 training images were created and used in order to train a CNN $e$ that maps the three parameters to the output, with very little loss (MSE of 0.1). 

A training set $\mathbf t$ of a similar size is collected by sampling in the same way. As a baseline method, we employ DCGAN~\cite{dcgan}, in which the generator function $G$ has four deconvolution layers (the open code of~\url{https://github.com/soumith/dcgan.torch} is used), and in which the input $x$ is a random vector in $[-1,1]^{100}$. The results are shown in Fig.~\ref{fig:poly}(b). While the generated images are similar to the class of generated polygons, they are not from this class and contain visible artifacts such as curved edges.

A TOS is then trained by minimizing Eq.~\ref{eq:Le} with the additional GAN constraints. The optimization minimizes $L_{c} + \alpha L_\text{GAN}$, for $\alpha=1$ ($L_{CONST}$ and $L_{TID}$ are irrelevant to this experiment), and with the input distribution $D_1$ of random vectors sampled uniformly in the $[-1,1]$ hypercube in 100D. The results, as depicted in Fig.~\ref{fig:poly}(c), show that TOS, which enjoys the additional supervision of $e$, produces results that better fit the polygon class.

\begin{figure}
\centering
\begin{tabular}{c@{~}c}
\raisebox{.5cm}{(a)} & \includegraphics[trim=0 68 400 0, clip, width=.92799\linewidth]{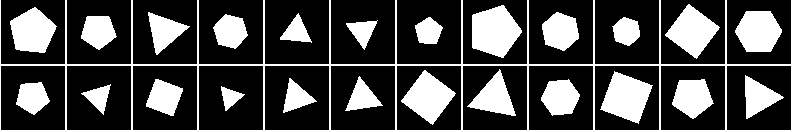}\\
\raisebox{.5cm}{(b)} & \includegraphics[trim=0 68 400 0, clip, width=.92799\linewidth]{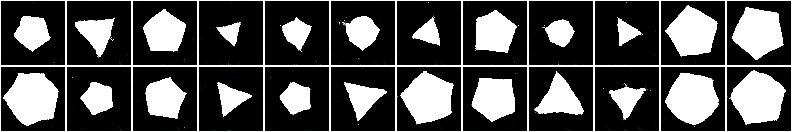}\\
\raisebox{.5cm}{(c)} & \includegraphics[trim=0 68 400 0, clip, width=.92799\linewidth]{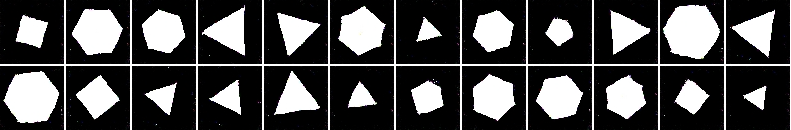}\\
\end{tabular}
\caption{\label{fig:poly} Toy problem. (a) Polygon images with three random parameters: number of vertices, radius of enclosing circle and rotation. (b) GAN generated images mimicking the class of polygon images. (c) Images created by TOS. The TOS is able to benefit from the synthesis engine $e$ and produces images that are noticeably more compliant than the GAN.}
\end{figure}

\subsection{Face Emoji}
\label{sec:faceemoji}
The proposed TOS method is evaluated for the task of generating specification-compliant emoji. In this task, we transfer an ``in-the-wild'' facial photograph to a set of parameters that defines an emoji. As the unlabeled training data of face images (domain $\mathcal X$), we use a set $\mathbf s$ of one million random images without identity information. The set $\mathbf t$ consists of assorted facial avatars ({\it emoji}) created by an online service (\url{bitmoji.com}). The emoji images were processed by an automatic process that detects, based on a set of heuristics, the center of the irises and the tip of the nose~\cite{02200}. Based on these coordinates, the emoji were centered and scaled into $152 \times 152$ RGB images.

The emoji engine of the online service is mostly additive. In order to train the TOS, we mimic it and have created a neural network $e$ that maps properties such as gender, length of hair, shape of eyes, etc. into an output image. The architecture is detailed in the appendix.

As the function $f$, we employ the representation layer of the DeepFace network~\cite{deepface}. This representation is 256-dimensional and was trained on a labeled set of four million images that does not intersect the set $\mathbf s$. Network $c$ maps a $64 \times 64 \time 3$ emoji to a configuration vector. It contains five convolutional layers, each followed by batch normalization and a leaky ReLU with a leakiness coefficient of 0.2. Network $g$ maps $f$'s representations to $64 \times 64$ RGB images. Following~\cite{02200}, this is done through a network with 9 blocks, each consisting of a convolution, batch-normalization and ReLU. The odd blocks 1,3,5,7,9 perform upscaling convolutions. The even ones perform $1 \times 1$ convolutions~\cite{nin}. Network $d$ takes $152 \times 152$ RGB images (either natural or scaled-up emoji) and consists of 6 blocks, each containing a convolution with stride 2, batch normalization, and a leaky ReLU. We set $\alpha = 0.01$, $\beta = 100$, $\gamma = 1$, $\delta = 0.0005$ as the tradeoff hyperparameters, 
after eyeballing the results of the first epoch of a very limited set of experiments.

For evaluation purposes only, we employ the benchmark of~\cite{02200}, which contains manually created emoji of $118$ random images from the CelebA dataset~\cite{celeba}. The benchmark was created by a team of professional annotators who used the web service that creates the emoji images.  Fig.~\ref{fig:IBtagB} shows side by side samples of the original image, the human generated emoji, the emoji generated by the generator function of DTN~\cite{02200}, and the emoji generated by both the generator $G=g\circ f$ and the compound generator $e\circ c \circ G$ of our TOS method.  As can be seen, the DTN emoji tend to be more informative, albeit less restrictive than the ones created manually. TOS respects the configuration space and creates emoji that are similar to the ones created by the human annotators, but which tend to carry more identity information. 

\begin{figure*}[t]
\begin{minipage}[c]{0.68\textwidth}
\begin{tabular}{c@{~~}c@{~~}c}
\includegraphics[trim=0 0 513 0, clip, height=12cm]{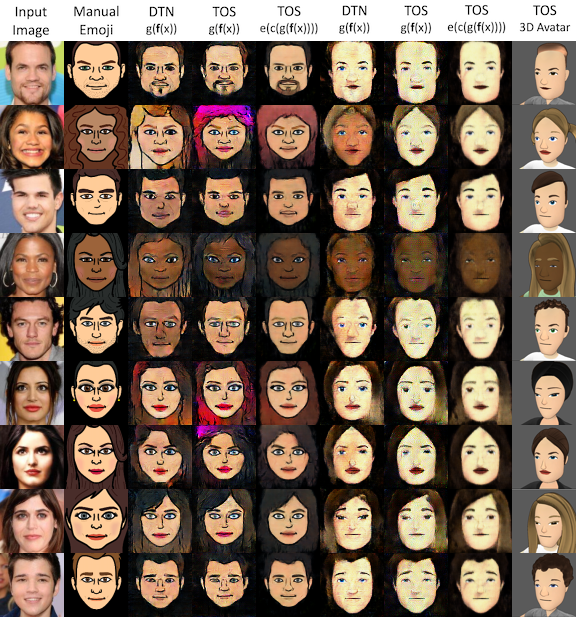}&
\includegraphics[trim=64 0 256 0, clip, height=12cm]{yaniv.png}&
\includegraphics[trim=320 0 0 0, clip, height=12cm]{yaniv.png}\\
{\small (a)} & {\small(b)} & {\small (c)}\\
\end{tabular}
  \end{minipage}\hfill
  \vspace{-.2cm}
  \begin{minipage}[c]{0.31\textwidth}
\renewcommand{\figurename}{$\leftarrow$ Figure}
\caption{\label{fig:IBtagB} Shown, side by side, are (a) sample images from the CelebA dataset. (b) emoji, from left to right: the images created manually using a web interface (for evaluation only), the result of DTN, and the two results of our TOS: $G(x)$ and then $e(c(G(x)))$. (c) VR avatar results: DTN, the two TOS results, and a 3D rendering of the resulting configuration file. See Tab.~\ref{tab:BvsG} for retrieval performance. The results of DANN~\cite{Ganin:2016:DTN:2946645.2946704} are not competitive and are shown in the appendix.}

\renewcommand{\figurename}{$\uparrow$Figure}
~\\

\noindent\begin{tabular}{c@{~}c}
\hspace{-.2cm}\includegraphics[trim=0 0 0 0, clip, width=0.492412348189\linewidth]{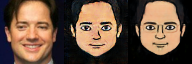}&
\hspace{-.1cm}\includegraphics[trim=0 0 0 0, clip, width=0.492412348189\linewidth]{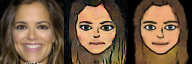}\\
\hspace{-.2cm}\includegraphics[trim=0 0 0 0, clip, width=0.492412348189\linewidth]{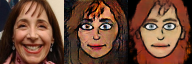}&
\hspace{-.1cm}\includegraphics[trim=0 0 0 0, clip, width=0.492412348189\linewidth]{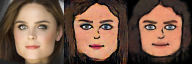}\\
\hspace{-.2cm}\includegraphics[trim=0 0 0 0, clip, width=0.492412348189\linewidth]{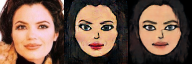}&
\hspace{-.1cm}\includegraphics[trim=0 0 0 0, clip, width=0.492412348189\linewidth]{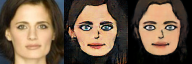}\\
\hspace{-.2cm}\includegraphics[trim=0 0 0 0, clip, width=0.492412348189\linewidth]{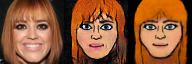}&
\hspace{-.1cm}\includegraphics[trim=0 0 0 0, clip, width=0.492412348189\linewidth]{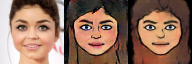}\\
\hspace{-.2cm}\includegraphics[trim=0 0 0 0, clip, width=0.492412348189\linewidth]{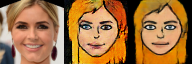}&
\hspace{-.1cm} \includegraphics[trim=0 0 0 0, clip, width=0.492412348189\linewidth]{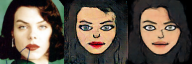}\\
\end{tabular}
\caption{\label{fig:fsc2} Multi-image results on Facescrub. Shown, side by side, are (i) the image selected to create the TOS and the DTN emoji, (ii) the DTN emoji, and (iii) the TOS emoji, obtained by $e\circ c\circ g \circ f$. See also appendix.}
\end{minipage}
\end{figure*}

\begin{table}[t]
\centering
\begin{tabular}{|l|cc|cc|}
\hline
\textbf{Method} & \multicolumn{2}{|c|}{{Emoji}}& \multicolumn{2}{c|}{Avatars}\\
& $\mathbf{g(f(x))}$ & $\mathbf{e(..(x))}$ & $\mathbf{g(f(x))}$ & $\mathbf{e(..(x))}$ \\
\hline
Manual & NA & 16,311 & NA & NA \\
DANN~\cite{Ganin:2016:DTN:2946645.2946704} & NA & 59,625 & NA & 52,435  \\
DTN~\cite{02200} & 16 & 18,079 & 195 & 38,805      \\
TOS & 30 & 3,519 & 758 & 11,153 \\
TOS fixed $\bar c$ & 26 & 14,990 & 253 & 43,160 \\
\hline
\end{tabular}
\caption{\label{tab:BvsG}Comparison of median rank for retrieval out of a set of 100,001 face images for either manually created emoji, or emoji and VR avatars created by DTN or TOS. Results are shown for the ``raw'' $G(x)$ as well as for the configuration compliant $e(..(x))$. Since DTN does not produce a configuration-compliant emoji, we obtain the results for the $e(..(x))$ column by applying to its output a pretrained network $\bar c$ that maps emoji to configurations. Also shown are DANN results obtained when training such a mapping $\bar c$ that is adapted to the samples in $\mathbf s$.}
\end{table}

In order to evaluate the identifiability of the resulting emoji, the authors of~\cite{02200} have collected a second example for each identity in the set of $118$ CelebA images and a set $\mathbf s'$ of 100,000 random face images (unsupervised, without identity), which were not included in $\mathbf s$. The VGG face CNN descriptor~\cite{Parkhi15} is then used in order to perform retrieval as follows. For each image $x$ in the manually annotated set, a gallery  $\mathbf s' \cup x'$ is created, where $x'$ is the other image of the person in $x$. Retrieval is then performed using VGG faces and either the manually created emoji, $G(x)$, or $e(c(G(x)))$ as the probe. 

In these experiments, the VGG network is used in order to avoid a bias that might be caused by using $f$ both for training the DTN and the TOS methods and for evaluation. The results are reported in Tab.~\ref{tab:BvsG}. As can be seen, the $G(x)$ emoji generated by DTN are extremely discriminative and obtain a median rank of 16 in cross-domain identification out of $10^5$ distractors. However, DTNs are not compatible with any configuration vector. In order to demonstrate this, we trained a network $\bar c$ that maps emoji images to configurations. 
When applied to the emoji generated by DTN and transforming the results, using $e$, back to an emoji, the obtained images are less identifiable than the emoji created manually (Tab.~\ref{tab:BvsG}, under $e(..(x))$). By comparison, the median rank of the emoji created by the configuration vector $c(G(x))$ of TOS is much better than the result obtained by the human annotators. As expected, DTN has more identifiable results than TOS when considering the output of $g(f(x))$ directly, since TOS has additional terms and the role of $L_\text{CONST}$ in TOS is naturally reduced.

The need to train $c$ and $G$ jointly, as is done in the TOS framework, is also verified in a second experiment, in which we fixed the network $c$ of TOS to be the pretrained network $\bar c$. 
The results of rendering the configuration vector were also not as good as those obtained by the unmodified TOS framework. As expected, querying by $G(x)$ directly, produces results that are between DTN and TOS.

It should be noted that using the pretrained $\bar c$ directly on inputs faces, leads to fixed configurations (modes), since $\bar c$ was trained to map from $\mathcal Y_1$ and not from $\mathcal X$. This is also true when performing the prediction based on $f$ mappings of the input and when training a mapping from $\mathcal X$ to $\mathcal Y_2$ under the $f$ distance on the resulting avatar. This situation calls for the use of unsupervised domain adaptation (Sec.~\ref{ref:problemformulation}) to learn a mapping from $\mathcal X$ to $\mathcal Y_2$ by adapting a mapping from $\mathcal Y_1$.  Despite some effort, applying the domain adaptation method of~\cite{Ganin:2016:DTN:2946645.2946704} did not result in satisfactory results (Tab.~\ref{tab:BvsG} and appendix). The best architecture found for this network follows the framework of domain-adversarial neural networks~\cite{Ganin:2016:DTN:2946645.2946704}. Our implementation consists of a feature network $p$ that resembles our network $c$ - with 4 convolution layers, a label predictor $l$ which consists of 3 fully connected layers, and a discriminative network $d$ that consists of 2 fully connected layers. The latter is preceded by a gradient reversal layer to ensure that the feature distributions of both domains are made similar. In both $l$ and $d$, each hidden layer is followed by batch normalization.

\noindent {\bf Human rating} Finally, we asked a group of 20 volunteers to select the better emoji, given a photo from celebA and two matching emoji: one created by the expert annotators and one created by TOS ($e\circ c\circ G$). The raters were told that they are presented with the results of two algorithms for automatically generating emoji and are requested to pick their favorable emoji for each image. The images were presented printed out, in random order, and the raters were given an unlimited amount of time. In 39.53\% of the answers, the TOS emoji was selected. This is remarkable considering that in a good portion of the celebA emoji, the TOS created very dark emoji in an unfitting manner (since $f$ is invariant to illumination and since the configuration has many more dark skin tones than lighter ones). 

TOS, therefore, not only provides more identifiable emoji, but is also very close to be on par with professional annotators. It is important to note that we did not compare to DTN in this rating, since DTN does not create a configuration vector,  which is needed for avatar applications (Fig~\ref{fig:piven}(b)).

\noindent{\bf Multiple Images Per Person} Following~\cite{02200}, we evaluate the results obtained per person and not just per image on the Facescrub dataset~\cite{facescrub}. For each person $q$, we considered the set of their images $X_q$, and selected the emoji that was most similar to their source image, i.e., the one for which:
$\argmin_{x \in X_q} || f(x)-f(e(c(G(x))))||$. The qualitative results are appealing and are shown in Fig.~\ref{fig:fsc}.

\vspace{-.1cm}
\subsection{VR Avatars}
\label{sec:socialvr}
\vspace{-.1cm}
We next apply the proposed TOS method to a commercial avatar generator engine, see Fig.~\ref{fig:IBtagB}(c). We sample random parameterizations and automatically align their frontally-rendered avatars into $64 \times 64$ RGB images to form the training set $\mathbf t$. We then train a CNN $e$ to mimic this engine and generate such images given their parameterization. Using the same architectures and configurations as in Sec.~\ref{sec:faceemoji}, including the same training set $\mathbf s$, we train $g$ and $c$ to map natural facial photographs to their engine-compliant set of parameters. We also repeat the same identification experiment and report median rankings of the analog experiments, see Tab.~\ref{tab:BvsG}(right). The 3D avatar engine is by design not as detailed as the 2D emoji one, with elements such as facial hair still missing and less part shapes available. In addition, the avatar model style is more generic and focused on real time puppeteering and not on cartooning. Therefore, the overall numbers are lower for all methods, as expected. TOS seems to be the only method that is able to produce identifiable configurations, while the other methods lead to ranking that is close to random.
\vspace{-.1cm}
\section{Conclusions}
\vspace{-.1cm}
\label{sec:conclusions}
With the advent of better computer graphics engines and the plethora of available models, and the ability of neural networks to compare cross-domain entities, the missing element for bridging between computer vision and computer graphics is the ability to link image data to a suitable parametrization. The previously presented DTN method showed a remarkable capability to create analogies without explicit supervision.  For example, highly identifiable emoji were generated. However, emoji applications call for parametrized characters, which can then be transformed by artists to other views and new expressions, and 
the emoji created by DTN cannot be converted to a configuration. The TOS method that we present is able to generate identifiable emoji that are coupled with a valid configuration vector. 

While TOS was presented in a way that requires the rendering function $e$ to be differentiable, working with black-box renderers using gradient estimation techniques is a common practice, e.g., in Reinforcement Learning.

{\small
\bibliographystyle{ieee}
\bibliography{gans}
}

\appendix
\section{Summary of Notations}

Tab.~\ref{tab:notation} itemizes the symbols used in this work. Fig.~2,3,4 of the main text illustrate many of these symbols.

\begin{table*}[t]
\centering
\begin{tabular}{|c|p{14cm}|}
\hline
Symbol & Meaning\\
\hline
\hline
$\mathcal X$ & Input space\\
$\mathcal Y$ & Output space\\
$\mathcal Y_1, \mathcal Y_2$ & Tied output spaces \\
$D_S$ & Source distribution \\
$D_T$ & Target distribution \\
$D_1$,$D_2$ & Input/output or other pairs of distributions \\
$\ell$ & A loss function, typically $\ell: \mathcal Y \times \mathcal Y \rightarrow \mathbb{R}_+$ \\
$\disc_{\mathcal{C}}(D_1,D_2)$ & Discrepancy between two distributions $D_1$ and $D_2$, using functions from the hypothesis class $\mathcal{C}$\\
$R_{D}[c_1,c_2]$ & The risk between two functions $c_1$ and $c_2$, i.e.,$\mathbb{E}_{x\sim D}\left[\ell(c_1(x),c_2(x))\right]$ \\
$y$ & The function from input to output we learn\\
$y_S, y_T$ & In domain adaptation, the source and target functions from input to output \\
$f$ & A pre-trained feature map\\
$e$ & A mapping from configurations ($\mathcal Y_2$) to parameterized outputs ($\mathcal Y_1$) \\
$c \in \mathcal H_3$ & A learned mapping from the space of parameterized outputs ($\mathcal Y_1$) to configurations ($\mathcal Y_2$) \\
$g \in \mathcal H_2$ & A generator function from the feature space (image of $f$) to the output space \\
$d$ & The discriminator of the GAN\\
$h \in \mathcal H$ & A mapping from input to output (different to each problem)\\
$G$ & A generator from the input space to the space $\mathcal Y_1$ \\
$L_\text{GAN}$ & The GAN loss term. Used together with Eq.~7 \\
$L_c$ & The loss that arises from the mismatch between $G$ and $c$. The specific form used is given in Eq.~5 \\
$\ell_e$ & For a given $x$, the mismatch between $G(x)$ and $e \circ c$ applied to it \\
$\mathbf s$ & The training set in $\mathcal X$ \\
$\mathbf t$& The training set in $\mathcal Y_1$ \\
$L_\text{CONST}$ & The term that enforces f-constancy \\
$L_\text{TID}$ & The term that enforces idempotency for the learned mapping  \\
$L_\text{TV}$ & Total Variation loss, which encourages smoothness of the output  \\
$\alpha$,$\beta$,$\gamma$,$\delta$ & Tradeoff parameters (weights in the loss term the network minimizes) \\
$p$ & The feature map of the DANN algorithm~\cite{Ganin:2016:DTN:2946645.2946704} \\
$l$ & The label predictor network of the DANN algorithm~\cite{Ganin:2016:DTN:2946645.2946704} \\
\hline
\end{tabular}
\caption{\label{tab:notation}The mathematical notations used in the paper.}
\end{table*}

\section{DANN results}

Fig.~\ref{fig:DANN} shows side by side samples of the original image and the emoji generated by the  method of~\cite{Ganin:2016:DTN:2946645.2946704}. As can be seen, these results do not preserve the identity very well, despite considerable effort invested in finding suitable architectures.  

\begin{figure*}[t]
\centering
\begin{tabular}{c@{~}c@{~}c@{~}c}
\includegraphics[trim=0 0 0 0, clip, width=0.22432932435\linewidth]{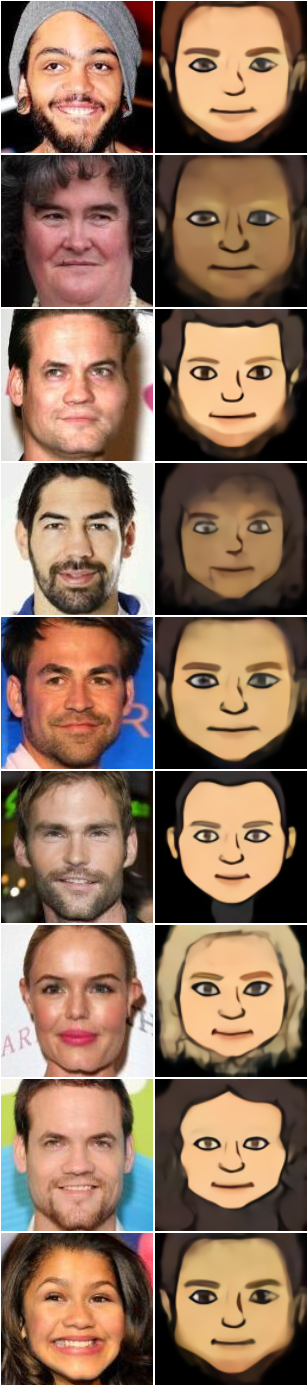}&
\includegraphics[trim=0 0 0 0, clip, width=0.22432932435\linewidth]{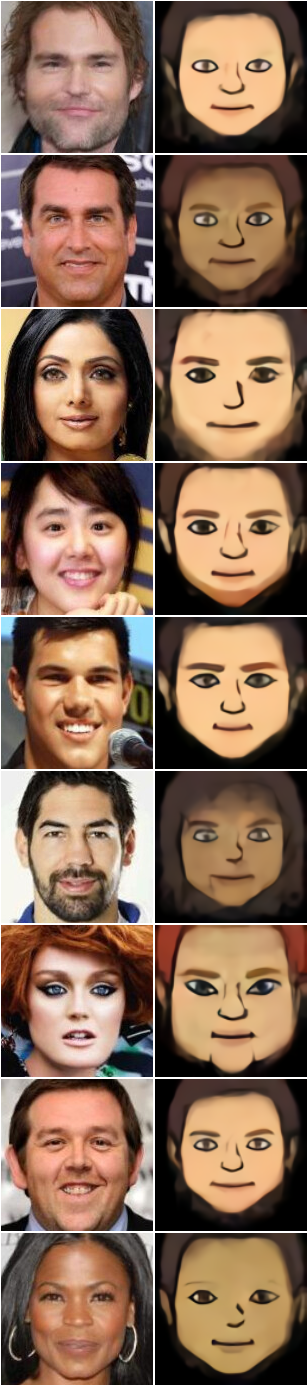}&
\includegraphics[trim=0 0 0 0, clip, width=0.22432932435\linewidth]{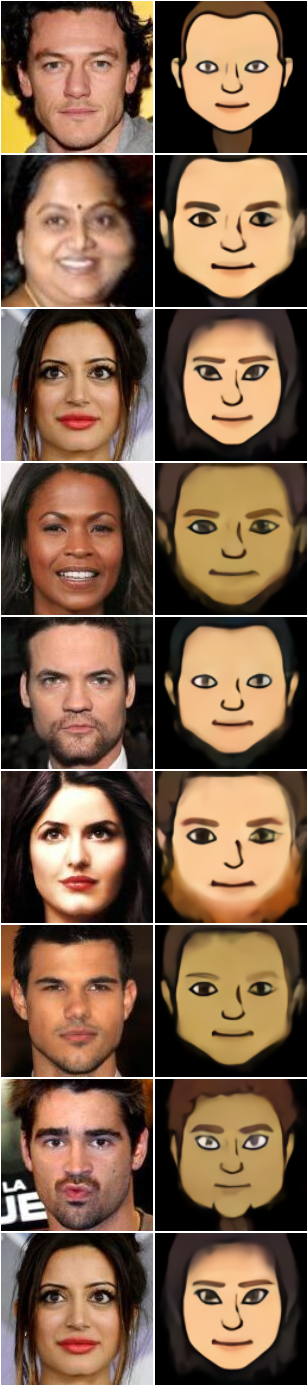}&
\includegraphics[trim=0 0 0 0, clip, width=0.22432932435\linewidth]{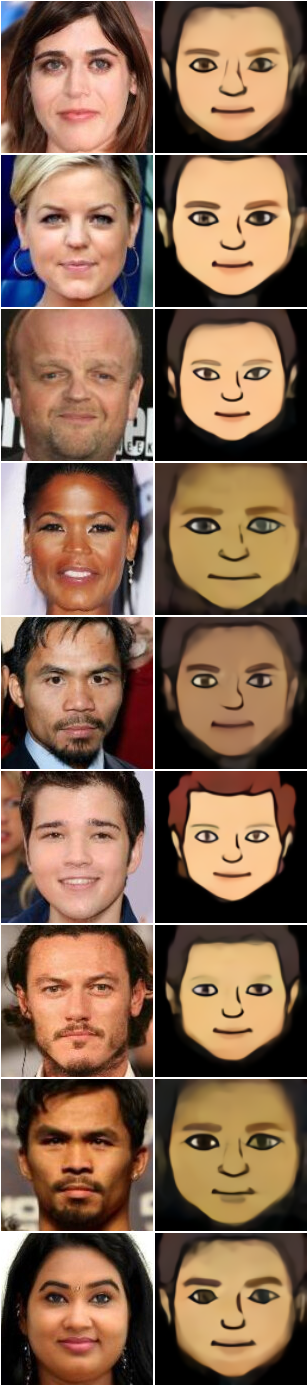}
\end{tabular}
\caption{\label{fig:DANN} Shown, side by side are sample images from the CelebA dataset and the results obtained by the DANN domain adaptation method~\cite{Ganin:2016:DTN:2946645.2946704}. These results are not competitive.}
\end{figure*}

\section{Multiple Images Per Person} 

Following~\cite{02200}, we evaluate the visual quality that is obtained per person and not just per image, by testing TOS on the Facescrub dataset~\cite{facescrub}. For each person $p$, we considered the set of their images $X_p$, and selected the emoji that was most similar to their source image, i.e., the one for which:
\begin{equation}
\label{eq:mmeq}
\argmin_{x \in X_p} || f(x)-f(e(c(G(x))))||.
\end{equation}

Fig.~\ref{fig:fsc2} depicts the results obtained by this selection method on sample images form the Facescrub dataset (it is an extension of Fig.~7 of the main text). The figure also shows, for comparison, the DTN~\cite{02200} result for the same image.

\begin{figure*}[t]
\centering
\begin{tabular}{c@{~}c@{~}c@{~}c}
\includegraphics[trim=0 0 0 0, clip, width=0.32412348189\linewidth]{facescrub/1.png}&
\includegraphics[trim=0 0 0 0, clip, width=0.32412348189\linewidth]{facescrub/101.png}&
\includegraphics[trim=0 0 0 0, clip, width=0.32412348189\linewidth]{facescrub/113.png}\\
\includegraphics[trim=0 0 0 0, clip, width=0.32412348189\linewidth]{facescrub/103.png}&
\includegraphics[trim=0 0 0 0, clip, width=0.32412348189\linewidth]{facescrub/117.png}&
\includegraphics[trim=0 0 0 0, clip, width=0.32412348189\linewidth]{facescrub/119.png}\\
\includegraphics[trim=0 0 0 0, clip, width=0.32412348189\linewidth]{facescrub/123.png}&
\includegraphics[trim=0 0 0 0, clip, width=0.32412348189\linewidth]{facescrub/125.png}&
\includegraphics[trim=0 0 0 0, clip, width=0.32412348189\linewidth]{facescrub/137.png}\\
\includegraphics[trim=0 0 0 0, clip, width=0.32412348189\linewidth]{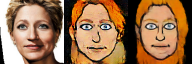}&
\includegraphics[trim=0 0 0 0, clip, width=0.32412348189\linewidth]{facescrub/143.png}&
\includegraphics[trim=0 0 0 0, clip, width=0.32412348189\linewidth]{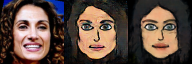}\\
\includegraphics[trim=0 0 0 0, clip, width=0.32412348189\linewidth]{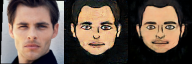}&
\includegraphics[trim=0 0 0 0, clip, width=0.32412348189\linewidth]{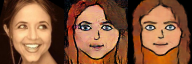}&
\includegraphics[trim=0 0 0 0, clip, width=0.32412348189\linewidth]{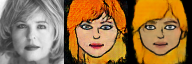}\\
\includegraphics[trim=0 0 0 0, clip, width=0.32412348189\linewidth]{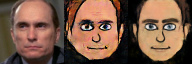}&
\includegraphics[trim=0 0 0 0, clip, width=0.32412348189\linewidth]{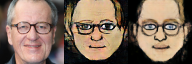}&
\includegraphics[trim=0 0 0 0, clip, width=0.32412348189\linewidth]{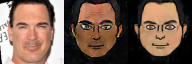}\\
\includegraphics[trim=0 0 0 0, clip, width=0.32412348189\linewidth]{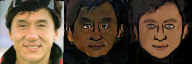}&
\includegraphics[trim=0 0 0 0, clip, width=0.32412348189\linewidth]{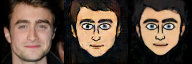}&
\includegraphics[trim=0 0 0 0, clip, width=0.32412348189\linewidth]{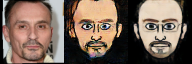}\\
\includegraphics[trim=0 0 0 0, clip, width=0.32412348189\linewidth]{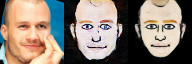}&
\includegraphics[trim=0 0 0 0, clip, width=0.32412348189\linewidth]{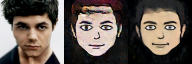}&
\includegraphics[trim=0 0 0 0, clip, width=0.32412348189\linewidth]{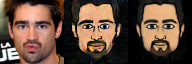}\\
\includegraphics[trim=0 0 0 0, clip, width=0.32412348189\linewidth]{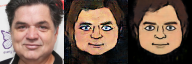}&
\includegraphics[trim=0 0 0 0, clip, width=0.32412348189\linewidth]{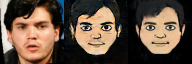}&
\includegraphics[trim=0 0 0 0, clip, width=0.32412348189\linewidth]{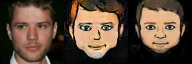}\\
\includegraphics[trim=0 0 0 0, clip, width=0.32412348189\linewidth]{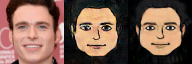}&
\includegraphics[trim=0 0 0 0, clip, width=0.32412348189\linewidth]{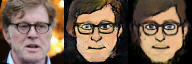}&
\includegraphics[trim=0 0 0 0, clip, width=0.32412348189\linewidth]{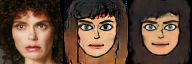}\\
\includegraphics[trim=0 0 0 0, clip, width=0.32412348189\linewidth]{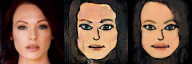}&
\includegraphics[trim=0 0 0 0, clip, width=0.32412348189\linewidth]{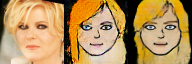}&
\includegraphics[trim=0 0 0 0, clip, width=0.32412348189\linewidth]{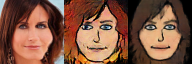}\\
\end{tabular}
\caption{\label{fig:fsc} The results obtained by the TOS method for a sample of individuals from the Facescrub dataset. Shown, side by side, are the image used to create the TOS and the DTN emoji, the DTN emoji, and the TOS emoji, obtained by $e\circ c\circ g \circ f$. The image that represents a person maximizes, out of all images for this person, $f$-constancy for the TOS method.}
\end{figure*}

\section{Detailed Architecture of the Various Networks}

In this section we describe the architectures of the networks used in for the emoji and avatar experiments. 

\subsection{TOS}

Network $g$ maps DeepFace's 256-dimensional representation~\cite{deepface} into $64 \times 64$ RGB emoji images. Following~\cite{02200}, this is done through a network with 9 blocks, each consisting of a convolution, batch-normalization and ReLU, except the last layer which employs Tanh activation. The odd blocks 1,3,5,7,9 perform upscaling convolutions with 512-256-128-64-3 filters respectively of spatial size $ 4\times 4$. The even ones perform $1 \times 1$ convolutions~\cite{nin}. The odd blocks use a stride of 2 and padding of 1, excluding the first one which does not use stride or padding. 

Network $e$ maps emoji parameterization into the matching $64 \times 64$ RGB emoji. The parameterization is given as binary vectors in $\mathbb{R}^{813}$ for emojis; Avatar parameterization is in $\mathbb{R}^{354}$. While there are dependencies among the various dimensions (an emoji cannot have two hairstyles at once), the binary representation is chosen for its simplicity and generality. $e$ is trained in a fully supervised way, using pairs of matching parameterization vectors and images in a supervised manner.

The architecture of $e$ employs five upscaling convolutions with 512-256-128-64-3 filters respectively, each of spatial size $4\times4$. All layers except the last one are batch normalized followed by a ReLU activation. The last layer is followed by Tanh activation, generating an RGB image with values in range $[-1,1]$. All the layers use a stride of 2 and padding of 1, excluding the first one which does not use stride or padding. 

Network $d$ takes $152 \times 152$ RGB images (either natural or scaled-up emoji) and outputs log-probabilities predicting if the image is fake or real. It consists of 6 blocks, each containing a convolution with stride 2, batch normalization, and a leaky ReLU with leakiness coefficient of 0.2. Each block contains 64-128-256-512-512-3 filters respectively. As before, the last layer does not employ batch normalization and ReLU.

Network $c$ maps a $64 \times 64 \time 3$ emoji to parameterization vector. It contains five convolutional layers, each followed by batch normalization and a leaky ReLU with a leakiness coefficient of 0.2. Each layer contains 64-128-256-512-813 filters respectively.  The last layer is followed by Tanh activation, generating a parameterization vector with values in range $[-1,1]$.

The networks used for the synthetic polygon experiment are somewhat simpler: $g$ has the same structure of as in the emoji experiment excluding the even convolutions i.e., it does not contain the $1 \times 1$ convolutions. The architecture of $d$ is unchanged. Finally, the architectures of $e$ and $c$ are updated to match the synthetic experiment parameterization. $e$ is changed to map a parameterization vectors in $\mathbb{R}^{3}$ to RGB images, and $c$ is trained to predict such a vector.
\subsection{DANN}

In the domain adaptation experiments, network $p$ extracts 2048-dimensional feature vectors from $64 \times 64$ RGB images. It resembles the structure of network $c$ - with 4 convolution layers. Each convolution is with 64-128-256-512 filters respectively. The last convolutional layer employs a stride of 1 instead of 2 and does not use batch-normalized or leaky ReLU. Finally, the network output is flattened to 1-dimensional feature vector.

The label prediction network $l$ accepts as input feature vectors generated by $p$ and outputs emoji parameterization vectors matching the input image. It consists of 3 fully connected layers. Each hidden layer is followed by batch-normalization and leaky ReLU activation. The last layer is followed by Tanh activation. The hidden layers contain 1024 and 512 units respectively.

The discriminator $d$ predicts the input image domain given its feature vector. It consists of two fully connected layers with 512 hidden units. The hidden layer is followed by batch normalization and leaky ReLU activations. It is preceded by a gradient reversal layer to ensure that the feature distributions of both domains are similar. The last layer is followed by Sigmoid activation, predicting the input image domain.

\end{document}